\documentclass[letter, 10pt, conference]{ieeeconf}
\IEEEoverridecommandlockouts
\usepackage{epsfig}
\usepackage{caption}
\usepackage{subcaption}
\usepackage{graphicx}
\usepackage{amsfonts}
\usepackage{amsmath}
\usepackage{amssymb}
\usepackage{color}
\usepackage{pstricks}
\usepackage{url}
\usepackage{cite}
\usepackage{longtable}
\usepackage{algorithm}
\usepackage{algorithmic}
\newtheorem{theorem}{Theorem}

\usepackage{tikz}

\usepackage[normalem]{ulem}

\title{Deep Model Reference Adaptive Control}

\author{Girish Joshi and Girish Chowdhary% <-this % stops a space
	\thanks{*Supported by the Laboratory Directed Research and Development program at Sandia National Laboratories, a multi-mission laboratory managed and operated by National Technology and Engineering Solutions of Sandia, LLC., a wholly owned subsidiary of Honeywell International, Inc., for the U.S. Department of Energy's National Nuclear Security Administration under contract DE-NA-0003525.}% <-this % stops a space
	\thanks{Authors are with Coordinated Science Laboratory, University of Illinois,
		Urbana-Champaign, IL, USA
		{\tt\small girishj2@illinois.edu,girishc@illinois.edu}}%
}

\begin{document}
	
	\maketitle \thispagestyle{empty} \pagestyle{empty}

\begin{abstract}

We present a new neuroadaptive architecture: Deep Neural Network based Model Reference Adaptive Control (DMRAC). Our architecture utilizes the power of deep neural network representations for modeling significant nonlinearities while marrying it with the boundedness guarantees that characterize MRAC based controllers. We demonstrate through simulations and analysis that DMRAC can subsume previously studied learning based MRAC methods, such as concurrent learning and GP-MRAC. This makes DMRAC a highly powerful architecture for high-performance control of nonlinear systems with long-term learning properties.

\end{abstract}

\section{Introduction}
Deep Neural Networks (DNN) have lately shown tremendous empirical performance in many applications and various fields such as computer vision, speech recognition, translation, natural language processing, Robotics, Autonomous driving and many more \cite{goodfellow2016deep}. Unlike their counterparts such as shallow networks with Radial Basis Function features \cite{sanner:TNN:92,liu2018gaussian}, deep networks learn features by learning the weights of nonlinear compositions of weighted features arranged in a directed acyclic graph \cite{2013arXiv1301.3605Y}. It is now pretty clear that deep neural networks are outshining other classical
machine-learning techniques\cite{hinton2012deep}.
Leveraging these successes, there have been many exciting new claims regarding the control of complex dynamical systems in simulation using deep reinforcement learning\cite{mnih2015human}.   However, Deep Reinforcement Learning (D-RL) methods typically do not guarantee stability or even the boundedness of the system during the learning transient. Hence despite significant simulation success, D-RL has seldomly been used in safety-critical applications. D-RL methods often make the ergodicity assumption, requiring that there is a nonzero probability of the system states returning to the origin. In practice, such a condition is typically enforced by resetting the simulation when a failure occurs. Unfortunately, however, real-world systems do not have this reset option. Unlike, D-RL much effort has been devoted in the field of adaptive control to ensuring that the system stays stable during learning.

Model Reference Adaptive Control (MRAC) is one such leading method for adaptive control that seeks to learn a high-performance control policy in the presence of significant model uncertainties \cite{ioannou1988theory,tao2003adaptive, Pomet_92_TAC}. The key idea in MRAC is to find an update law for a parametric model of the uncertainty that ensures that the candidate Lyapunov function is non-increasing. Many update laws have been proposed and analyzed, which include but not limited to $\sigma$-modification \cite{ioannou1996robust}, $e$-modification \cite{annaswamy_CDC_89}, and projection-based updates \cite{Pomet_92_TAC}. More modern laws extending the classical parametric setting include $\ell_1$-adaptive control \cite{hovakimyan2010ℒ1} and concurrent learning \cite{chowdhary2013concurrent} have also been studied. 

A more recent work introduced by the author is the Gaussian Process Model Reference Adaptive Control (GP-MRAC), which utilizes a GP as a model of the uncertainty. A GP is a Bayesian nonparametric adaptive element that can adapt both its weights and the structure of the model in response to the data. The authors and others have shown that GP-MRAC has strong long-term learning properties as well as high control performance \cite{chowdhary2015bayesian, joshi2018adaptive}. However, GPs are  ``shallow'' machine learning models, and do not utilize the power of learning complex features through compositions as deep networks do (see \ref{sec:feature}). Hence, one wonders whether the power of deep learning could lead to even more powerful learning based MRAC architectures than those utilizing GPs.
% GPs are a class of machine learning models that build their features through comparison with a select dictionary of reference points (see section \ref{sec:feature}). 

In this paper, we address this critical question: How can MRAC utilize deep networks while guaranteeing stability? Towards that goal, our contributions are as follows: a) We develop an MRAC architecture that utilizes DNNs as the adaptive element; b) We propose an algorithm for the online update of the weights of the DNN by utilizing a dual time-scale adaptation scheme. In our algorithm, the weights of the outermost layers are adapted in real time, while the weights of the inner layers are adapted using batch updates c) We develop theory to guarantee Uniform Ultimate Boundedness (UUB) of the entire DMRAC controller; d) We demonstrate through simulation results that this architecture has desirable long term learning properties. 

We demonstrate how DNNs can be utilized in stable learning schemes for adaptive control of safety-critical systems. This provides an alternative to deep reinforcement learning for adaptive control applications requiring stability guarantees. Furthermore, the dual time-scale analysis scheme used by us should be generalizable to other DNN based learning architectures, including reinforcement learning.

\section{Background} 
\subsection{Deep Networks and Feature spaces in machine learning}\label{sec:feature} 
The key idea in machine learning is that a given function can be encoded with weighted combinations of  \textit{feature} vector $\Phi \in \mathcal{F}$, s.t $\Phi(x)=[\phi_1(x), \phi_2(x),...,\phi_k(x)]^T\in \mathbb{R}^k$, and $W^*\in\mathbb{R}^{k \times m}$ a vector of `ideal' weights s.t $\|y(x)-W^{*^T}\Phi(x)\|_\infty<\epsilon(x)$.
Instead of hand picking features, or relying on polynomials, Fourier basis functions, comparison-type features used in support vector machines \cite{schoelkofp:01,scholkopf2002learning} or Gaussian Processes \cite{rasmussen2006gaussian}, DNNs utilize composite functions of features arranged in a directed acyclic graphs, i.e. $\Phi(x)=\phi_n(\theta_{n-1},\phi_{n-1}(\theta_{n-2},\phi_{n-2}(...))))$ 
where $\theta_i$'s are the layer weights. The universal approximation property of the DNN with commonly used feature functions such as sigmoidal, tanh, and RELU is proved in the work by Hornik's \cite{Hornik:NN89} and shown empirically to be true by recent results \cite{2016arXiv160300988M,Poggio2017,2016arXiv161103530Z}. 
Hornik et al. argued the network with at least one hidden layer (also called Single Hidden Layer (SHL) network) to be a universal approximator. However, empirical results show that the networks with more hidden layers show better generalization capability in approximating complex function. While the theoretical reasons behind better generalization ability of DNN are still being investigated \cite{2016Matus}, for the purpose of this paper, we will assume that it is indeed true, and focus our efforts on designing a practical and stable control scheme using DNNs.

\subsection{Neuro-adaptive control}
Neural networks in adaptive control have been studied for a very long time. The seminal paper by Lewis \cite{Lewis:AJC:99} utilized Taylor series approximations to demonstrate uniform ultimate boundedness with a single hidden neural network. SHL networks are nonlinear in the parameters; hence, the analysis previously introduced for linear in parameter, radial basis function neural networks introduced by Sanner and Slotine does not directly apply \cite{sanner:TNN:92}. The back-propagation type scheme with non-increasing Lyapunov candidate as a constraint, introduced in Lewis' work has been widely used in Neuro-adaptive MRAC. Concurrent Learning MRAC (CL-MRAC) is a method for learning based neuro-adaptive control developed by the author to improve the learning properties and provide exponential tracking and weight error convergence guarantees. However, similar guarantees have not been available for SHL networks. There has been much work, towards including deeper neural networks in control; however, strong guarantees like those in MRAC on the closed-loop stability during online learning are not available. In this paper, we propose a dual time-scale learning approach which ensures such guarantees. Our approach should be generalizable to other applications of deep neural networks, including policy gradient Reinforcement Learning (RL) \cite{sutton1992reinforcement} which is very close to adaptive control in its formulation and also to more recent work in RL for control \cite{modares2014integral}. 
 
\subsection{Stochastic Gradient Descent and Batch Training}
We consider a deep network model with parameters $\boldsymbol{\theta}$, and consider the problem of optimizing a non convex loss function $L(\boldsymbol{Z, \theta})$, with respect to $\boldsymbol{\theta}$. Let $L(\boldsymbol{Z, \theta})$ is defined as average loss over $M$ training sample data points.
 \begin{equation}
     L(\boldsymbol{Z, \theta}) = \frac{1}{M}\sum_{i=1}^M \ell(\boldsymbol{Z_i, \theta})
     \label{empirical_loss}
 \end{equation}
where $M$ denotes the size of sample training set. For each sample size of $M$, the training data are in form of $M$-tuple $Z^M = (Z_1, Z_2, \dots Z_M)$ of $Z-$valued random variables drawn according to some unknown distribution $P \in \mathcal{P}$. Where each $Z_i = \{x_i, y_i\}$ are the labelled pair of input and target values. For each $P$ the expected loss can be computed as $\boldsymbol{E}_p(\ell(\boldsymbol{Z,\theta}))$. The above empirical loss \eqref{empirical_loss} is used as proxy for the expected value of loss with respect to the true data generating distribution.
 
Optimization based on the Stochastic Gradient Descent (SGD) algorithm uses a stochastic approximation of the gradient of the loss $L(\boldsymbol{Z, \theta})$ obtained over a mini-batch of $M$ training examples drawn from buffer $\mathcal{B}$. The resulting SGD weight update rule
 \begin{eqnarray}
 \boldsymbol{\theta}_{k+1} &=& \boldsymbol{\theta}_k - \eta
 \frac{1}{M}\sum_{i=1}^M \nabla_{ \boldsymbol{\theta}}\ell(\boldsymbol{Z_i, \theta}_k)
 \label{SGD2}
 \end{eqnarray}
where $\eta$ is the learning rate. Further details on generating i.i.d samples for DNN learning and the training details of network are provided in section \ref{section_DMRAC}.

\section{System Description}
\label{system_description}
This section discusses the formulation of model reference adaptive control (see e.g. \cite{ioannou1988theory}). We consider the following system with  uncertainty $\Delta(x)$:
\begin{equation}
\dot x(t) = Ax(t) + B(u(t) + \Delta(x))
\label{eq:0}
\end{equation}
where $x(t) \in \mathbb{R}^n$, $t \geqslant 0$ is the state vector, $u(t) \in \mathbb{R}^m$, $t \geqslant 0$ is the control input, $A \in \mathbb{R}^{n \times n}$, $B \in \mathbb{R}^{n \times m}$ are known system matrices and we assume the pair $(A,B)$ is controllable. The term $\Delta(x) : \mathbb{R}^n \to \mathbb{R}^m$ is matched system uncertainty and be Lipschitz continuous in $x(t) \in \mathcal{D}_x$. Let $ \mathcal{D}_x \subset \mathbb{R}^n$ be a compact set and the control $u(t)$ is assumed to belong to a set of admissible control inputs of measurable and bounded functions, ensuring the existence and uniqueness of the solution to \eqref{eq:0}.

The reference model is assumed to be linear and therefore the desired transient and steady-state performance is defined by a selecting the system eigenvalues in the negative half plane. The desired closed-loop response of the reference system is given by
\begin{equation}
\dot x_{rm}(t) = A_{rm}x_{rm}(t) + B_{rm}r(t)
\label{eq:ref model}
\end{equation}
where $x_{rm}(t) \in \mathcal{D}_x  \subset \mathbb{R}^{n}$ and $A_{rm} \in \mathbb{R}^{n \times n}$ is Hurwitz and $B_{rm} \in \mathbb{R}^{n \times r}$. Furthermore, the command $r(t) \in \mathbb{R}^{r}$ denotes a bounded, piece wise continuous, reference signal and we assume the reference model (\ref{eq:ref model}) is bounded input-bounded output (BIBO) stable \cite{ioannou1988theory}.

The true uncertainty $\Delta(x)$ in unknown, but it is assumed to be continuous over a compact domain $\mathcal{D}_x \subset \mathbb{R}^n$. A Deep Neural Networks (DNN) have been widely used to represent a function when the basis vector is not known.  Using DNNs, a non linearly parameterized network estimate of the uncertainty can be written as $\hat \Delta(x) \triangleq \theta_n^T\Phi(x)$,  
% \begin{equation}
% \hat \Delta(x) \triangleq \hat W^T\phi(x)
% \label{eq:nn_estimate_defn}
% \end{equation}
where $\theta_n \in \mathbb{R}^{k \times m}$ are network weights for the final layer and $\Phi(x)=\phi_n(\theta_{n-1},\phi_{n-1}(\theta_{n-2},\phi_{n-2}(...))))$, is a $k$ dimensional feature vector which is function of inner layer weights, activations and inputs. The basis vector $\Phi(x) \in \mathcal{F}: \mathbb{R}^{n} \to \mathbb{R}^{k}$ is considered to be Lipschitz continuous to ensure the existence and uniqueness of the solution (\ref{eq:0}).
%\gXX{is this description supposed to also include Deep networks? I think it would be a good idea to just mention how the basis can either be radial basis or created via a conjuction of bases as in deep networks, but here we focus on single layer networks to facilitate analysis}
\subsection{Total Controller}
\label{adaptive_identification}
The aim is to construct a feedback law $u(t)$, $t \geqslant 0$, such that the state of the uncertain dynamical system (\ref{eq:0}) asymptotically tracks the state of the reference model (\ref{eq:ref model}) despite the presence of matched uncertainty.

A tracking control law consisting of linear feedback term $u_{pd} = Kx(t)$, a linear feed-forward term $u_{crm} = K_rr(t)$ and an adaptive term $\nu_{ad}(t)$ form the total controller 
%\vspace{-2mm}
\begin{equation}
u = u_{pd} + u_{crm} - \nu_{ad}
\label{eq:total_Controller}
\end{equation}
The baseline full state feedback and feed-forward controller is designed to satisfy the matching conditions such that $A_{rm} = A-BK$ and $B_{rm} = BK_r$. For the adaptive controller ideally we want $\nu_{ad}(t) = \Delta(x(t))$. Since we do not have true uncertainty information, we use a DNN estimate of the system uncertainties in the controller as $\nu_{ad}(t) = \hat{\Delta}(x(t))$.

\subsection{Deep Model Reference Generative Network (D-MRGEN) for uncertainty estimation} 
Unlike traditional MRAC or SHL-MRAC weight update rule, where the weights are moved in the direction of diminishing tracking error, training a deep Neural network is much more involved.  Feed-Forward networks like DNNs are trained in a supervised manner over a batch of i.i.d data. Deep learning optimization is based on Stochastic Gradient Descent (SGD) or its variants. The SGD update rule relies on a stochastic approximation of the expected value of the gradient of the loss function over a training set or mini-batches.

To train a deep network to estimate the system uncertainties, unlike MRAC we need labeled pairs of state-true uncertainties $\{x(t),\Delta(x(t))\}$ i.i.d samples. Since we do not have access to true uncertainties ($\Delta(x)$), we use a generative network to generate estimates of $\Delta(x)$  to create the labeled targets for deep network training. For details of the generative network architecture in the adaptive controller, please see \cite{joshi2018adaptive}. This generative network is derived from separating the DNN into inner feature layer and the final output layer of the network. We also separate in time-scale the weight updates of these two parts of DNN. Temporally separated weight update algorithm for the DNN, approximating system uncertainty is presented in more details in further sections.

\subsection{Online Parameter Estimation law}
\label{Identification}
The last layer of DNN with learned features from inner layer forms the Deep-Model Reference Generative Network (D-MRGeN).  We use the MRAC learning rule to update pointwise in time, the weights of the D-MRGeN in the direction of achieving asymptotic tracking of the reference model by the actual system.

Since we use the D-MRGeN estimates to train DNN model, we first study the admissibility and stability characteristics of the generative model estimate ${\Delta}'(x)$ in the controller (\ref{eq:total_Controller}). To achieve the asymptotic convergence of the reference model tracking error to zero, we use the D-MRGeN estimate in the controller \eqref{eq:total_Controller} as $\nu_{ad} = \Delta'(x)$
\begin{equation}
    \nu_{ad}(t) = W^T\phi_n(\theta_{n-1},\phi_{n-1}(\theta_{n-2},\phi_{n-2}(...))))
\end{equation}
To differentiate the weights of D-MRGeN from last layer weights of DNN ``$\theta_n$",  we denote D-MRGeN weights as ``$W$". 

\textbf{\emph{Assumption 1:}} Appealing to the universal approximation property of Neural Networks \cite{park1991universal} 
we have that, for every given basis functions $\Phi(x) \in \mathcal{F}$ there exists unique ideal weights $W^* \in \mathbb{R}^{k \times m}$ and $\epsilon_1(x) \in \mathbb{R}^{m}$ such that the following
approximation holds
\begin{equation}
\Delta(x) = W^{*T}\Phi(x) + \epsilon_1(x), \hspace{2mm} \forall x(t) \in \mathcal{D}_x \subset \mathbb{R}^{n}
\label{eq:3}
\end{equation}
\textbf{\emph{Fact 1:}} The network approximation error $\epsilon_1(x)$ is upper bounded, s.t  $\bar{\epsilon}_1 = \sup_{x \in \mathcal{D}_x}\|\epsilon_1(x)\|$, and can be made arbitrarily small given sufficiently large number of basis functions.

The reference model tracking error is defined as $e(t) = x_{rm}(t)- x(t)$.
Using (\ref{eq:0}) \& (\ref{eq:ref model}) and the controller of form (\ref{eq:total_Controller}) with adaptation term $\nu_{ad}$, the tracking error dynamics can be written as
\begin{equation}
\dot e(t) = \dot x_{rm}(t) - \dot{x}(t)
\label{eq:13}
\end{equation}
\begin{equation}
\dot e(t) = A_{rm}e(t) + \tilde W^T\Phi(x) +  \epsilon_1(x)
\label{eq:14}
\end{equation}
where $\tilde W = W^*-W$ is error in parameter.

The estimate of the unknown true network parameters $W^*$ are calculated on-line using the weight update rule (\ref{eq:18}); correcting the weight estimates in the direction of minimizing the instantaneous tracking error $e(t)$.  The resulting update rule for network weights in estimating the total uncertainty in the system is as follows 
\begin{equation}
\dot {W} = \Gamma proj(W,\Phi(x)e(t)'P) \hspace{5mm} {W}(0) = {W}_0 \label{eq:18}
\end{equation} 
where $\Gamma \in \mathbb{R}^{k \times k}$ is the learning rate and $P \in \mathbb{R}^{n \times n}$ is a positive definite matrix. For given Hurwitz $A_{rm}$, the matrix $P \in \mathbb{R}^{n \times n}$ is a positive definite solution of Lyapunov equation $A_{rm}^TP + PA_{rm} + Q = 0$ for given $Q > 0$
% \begin{equation}
% A_{rm}^TP + PA_{rm} + Q = 0
% \end{equation}

\textbf{\emph{Assumption 2:}} For uncertainty parameterized by unknown true weight ${W}^* \in \mathbb{R}^{k \times m}$ and known nonlinear basis $\Phi(x)$,
the ideal weight matrix is assumed to be upper bounded s.t $\|{W}^*\| \leq \mathcal{W}_b$. This is not a restrictive assumption. 

\subsubsection{Lyapunov Analysis}
The on-line adaptive identification law (\ref{eq:18}) guarantees the asymptotic convergence of the tracking errors $e(t)$ and parameter error $\tilde W(t)$ under the condition of persistency of excitation \cite{aastrom2013adaptive,ioannou1988theory} for the structured uncertainty. Similar to the results by Lewis for SHL networks \cite{lewis1999nonlinear}, we show here that under the assumption of unstructured uncertainty represented by a deep neural network, the tracking error is uniformly ultimately bounded (UUB). We will prove the following theorem under switching feature vector assumption.
\begin{theorem}
    Consider the actual and reference plant model (\ref{eq:0}) \& (\ref{eq:ref model}). If the weights parameterizing total uncertainty in the system are updated according to identification law \eqref{eq:18} Then the tracking error $\|e\|$ and error in network weights $\|\tilde W\|$ are bounded for all $\Phi \in \mathcal{F}$.
\label{Theorem-1}
\end{theorem} 
\begin{proof}
The feature vectors belong to a function class characterized by the inner layer network weights $\theta_i$ s.t $\Phi \in \mathcal{F}$. We will prove the Lyapunov stability under the assumption that inner layer of DNN presents us a feature which results in the worst possible approximation error compared to network with features before switch. 

For the purpose of this proof let $\Phi(x)$ denote feature before switch and $\bar{\Phi}(x)$ be the feature after switch. We define the error $\epsilon_2(x)$ as,
\begin{equation}
    \epsilon_2(x) = \sup_{\bar{\Phi} \in \mathcal{F}}\left|W^T\bar{\Phi}(x) - W^T\Phi(x)\right|
\end{equation}
Similar to \textbf{\emph{Fact-1}} we can upper bound the error $\epsilon_2(x)$ as $\bar{\epsilon}_2 = \sup_{x \in \mathcal{D}_x}\|\epsilon_2(x)\|$.
By adding and subtracting the term $W^T\bar{\Phi}(x)$, we can rewrite the error dynamics \eqref{eq:14} with switched basis as,
\begin{eqnarray}
\dot e(t) &=& A_{rm}e(t) + W^{*T}\Phi(x) - W^T\Phi(x) \nonumber \\
&& + W^T\bar{\Phi}(x) - W^T\bar{\Phi}(x) + \epsilon_1(x)
\label{eq:14_1}
\end{eqnarray}
From \textbf{\emph{Assumption-1}} we know there exists a $W^*$ $\forall \Phi \in \mathcal{F}$. Therefore we can replace $W^{*T}\Phi(x)$ by $W^{*T}\bar{\Phi}(x)$ and rewrite the Eq-\eqref{eq:14_1} as
\begin{eqnarray}
\dot e(t) &=& A_{rm}e(t) + \tilde{W}^{T}\bar{\Phi}(x) + W^T(\bar{\Phi}(x) - \Phi(x)) + \epsilon_1(x) \nonumber \\
\label{eq:14_2}
\end{eqnarray}
For arbitrary switching, for any $\bar{\Phi}(x) \in \mathcal{F}$, we can prove the boundedness by considering worst possible approximation error and therefore can write,
\begin{eqnarray}
\dot e(t) &=& A_{rm}e(t) + \tilde{W}^{T}\bar{\Phi}(x) +  \epsilon_2(x) + \epsilon_1(x) 
\label{eq:14_3}
\end{eqnarray}
Now lets consider $V(e,\tilde W) > 0$ be a differentiable, positive definite radially unbounded Lyapunov candidate function,
%ensuring the asymptotic convergence of tracking error $e(t)$ and error in network weight $\tilde W$
\begin{equation}
V(e,\tilde W) = e^TPe + \frac{\tilde W^T \Gamma^{-1} \tilde W}{2}
\label{eq:20}
\end{equation}
%where $\Gamma >0$ is the adaption rate.
The time derivative of the lyapunov function (\ref{eq:20}) along the trajectory (\ref{eq:14_3}) can be evaluated as
\begin{equation}
\dot V(e,\tilde W) = \dot e^TPe + e^TP \dot e - \tilde W^T\Gamma^{-1}\dot{\hat W}
\label{eq:25}
\end{equation}
% \begin{eqnarray}
% \dot V(e,\tilde W) &=& -e^TQe + \left(\phi(x)e'P - \Gamma^{-1}\dot{\hat W}\right)2\tilde W \nonumber\\
% && + 2e^TP\epsilon(x) 
% \label{eq:21}
% \end{eqnarray}
Using (\ref{eq:14_3}) \& (\ref{eq:18}) in (\ref{eq:25}), the time derivative of the lyanpunov function reduces to 
\begin{eqnarray}
\dot V(e,\tilde W) &=& -e^TQe + 2e^TP\epsilon(x)
\label{eq:22}
\end{eqnarray}
where $\epsilon(x) = \epsilon_1(x) +\epsilon_2(x)$ and $\bar{\epsilon} = \bar{\epsilon_1} + \bar{\epsilon_2}$.\\ 
Hence $\dot V(e,\tilde W) \leq 0$ outside compact neighborhood of the origin $e = 0$, for some sufficiently large $\lambda_{min}(Q)$. 
\begin{equation}
\|e(t)\| \geq \frac{2\lambda_{max}(P)\bar\epsilon}{\lambda_{min}(Q)}
\label{eq:error_bound}
\end{equation}
Using the BIBO assumption $x_{rm}(t)$ is bounded for bounded reference signal $r(t)$, thereby $x(t)$ remains bounded.  Since $V(e,\tilde W)$  is radially unbounded the result holds for all $x(0) \in \mathcal{D}_x$.
Using the fact, the error in parameters $\tilde{W}$ are bounded through projection operator \cite{larchev2010projection} and further using Lyapunov theory and Barbalat’s Lemma \cite{narendra2012stable} we can show that $e(t)$ is uniformly ultimately bounded in vicinity to zero solution. 
\end{proof}

From Theorem-\ref{Theorem-1} \& (\ref{eq:14}) and using system theory \cite{kailath1980linear} we can infer that as $e(t) \to 0$, ${\Delta}'(x) \to \Delta(x)$ in point-wise sense. Hence D-MRGeN estimates $y_{\tau} = {\Delta}'(x_{\tau})$ are admissible target values for training DNN features over the data $Z^M = \{\{x_{\tau},y_{\tau}\}\}_{\tau = 1}^{M}$. 

The details of DNN training and implementation details of DMRAC controller is presented in the following section.

\section{Adaptive Control using Deep nets (DMRAC)}
\label{section_DMRAC}
The DNN architecture for MRAC is trained in two steps. We separate the DNN into two networks, as shown in Fig-\ref{DNN_architecture}. The faster learning outer adaptive network and slower deep feature network. DMRAC learns underlying deep feature vector to the system uncertainty using locally exciting uncertainty estimates obtained using a generative network. Between successive updates of the inner layer weights, the feature provided by the inner layers of the deep network is used as the fixed feature vector for outer layer adaptive network update and evaluation. The algorithm for DNN learning and DMRAC controller is provided in Algorithm-\ref{alg:DMRAC}. Through this architecture of mixing two-time scale learning, we fuse the benefits of DNN memory through the retention of relevant, exciting features and robustness, boundedness guarantee in reference tracking. This key feature of the presented framework ensures robustness while guaranteeing long term learning and memory in the adaptive network.

Also as indicated in the controller architecture Fig-\ref{DNN_architecture} we can use contextual state `$c_i$' other than system state $x(t)$ to extract relevant features. These contextual states could be relevant model information not captured in system states. For example, for an aircraft system, vehicle parameters like pitot tube measurement, the angle of attack, engine thrust, and so on. These contextual states can extract features which help in decision making in case of faults. The work on DMRAC with contextual states will be dealt with in the follow on work.
\begin{figure}
    \centering
    \includegraphics[width=1.0\linewidth]{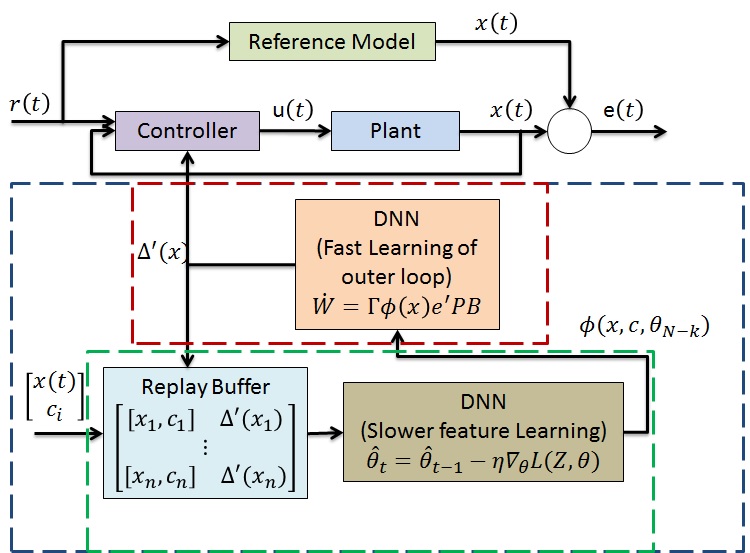}
    \caption{DMRAC training and controller details}
    \label{DNN_architecture}
    \vspace{-5mm}
\end{figure}
The DNN in DMRAC controller is trained over training dataset $Z^M = \{x_i, {\Delta}'(x_i)\}_{i=1}^M$, where the ${\Delta}'(x_i)$ are D-MRGeN estimates of the uncertainty. The training dataset $Z^M$ is randomly drawn from a larger data buffer $\mathcal{B}$. Not every pair of data $\{x_i, {\Delta}'(x_i)\}$ from D-MRGeN is added to the training buffer $\mathcal{B}$. We qualify the input-target pair based on kernel independence test such that to ensure that we collect locally exciting independent information which provides a sufficiently rich representation of the operating domain. Since the state-uncertainty data is the realization of a Markov process, such a method for qualifying data to be sufficiently independent of previous data-points is necessary. The algorithm details to qualify and add a data point to the buffer is provided in detail in subsection \ref{subsection-buffer}.

\subsection{Details of Deep Feature Training using D-MRGeN}
This section provides the details of the DNN training over data samples observed over n-dimensional input subspace $x(t) \in \mathcal{X} \in \mathbb{R}^{n}$ and m-dimensional targets subspace $y\in\mathcal{Y} \in \mathbb{R}^m$. The sample set is denoted as $\mathcal{Z}$ where $\mathcal{Z} \in \mathcal{X} \times \mathcal{Y}$.

We are interested in the function approximation tasks for DNN. The function $f_{\boldsymbol{\theta}}$ is the learned approximation to the model uncertainty with parameters $\boldsymbol{\theta} \in \boldsymbol{\Theta}$, where $\boldsymbol{\Theta}$ is the space of parameters, i.e. $f_{\boldsymbol{\theta}}: \mathbb{R}^n \to \mathbb{R}^m$. We assume a training data buffer $\mathcal{B}$ has $p_{max}$ training examples, such that the set $Z^{p_{max}} = \{Z_i | Z_i \in \mathcal{Z}\}_{i=1}^{p_{max}} = \{(x_i, y_i) \in \mathcal{X}\times\mathcal{Y}\}_{i=1}^{p_{max}}$. The samples are independently drawn from the buffer $\mathcal{B}$ over probability distribution $P$. The hypothesis set, which consist of all possible functions $f_{\boldsymbol{\theta}}$ is denoted as $\mathcal{H}$. Therefore a learning algorithm $\mathcal{A}$ (in our case SGD) is a mapping from $\mathcal{A}: \mathcal{Z}^{p_{max}} \to \mathcal{H}$ 

The loss function, which measures the discrepancy between true target $y$ and algorithm's estimated target function value $f_{\boldsymbol{\theta}}$ is denoted by $L(y, f_{\boldsymbol{\theta}}(x))$. Specific to work presented in this paper, we use a $\ell_2$-norm between values i.e. $\mathbb{E}_p(\ell(y, f_{\boldsymbol{\theta}}(x))) = \mathbb{E}_P \left(\|y_i - f_{\boldsymbol{\theta}}(x_i)\|_2\right)$ as loss function for DNN training. The empirical loss \eqref{empirical_loss} is used to approximate the loss function since the distribution $P$ is unknown to learning algorithm. The weights are updated using SGD in the direction of negative gradient of the loss function as given in \eqref{SGD2}.

Unlike the conventional DNN training where the true target values $y \in  \mathcal{Y}$ are available for every input $x \in  \mathcal{X}$, in DMRAC true system uncertainties as the labeled targets are not available for the network training. We use the part of the network itself (the last layer) with pointwise weight updated according to MRAC-rule as the generative model for the data. The D-MRGeN uncertainty estimates $y = {W}^T\Phi(x,\theta_1,\theta_2, \ldots \theta_{n-1})= {\Delta}'(x)$ along with inputs $x_i$ make the training data set $Z^{p_{max}} = \{x_i, {\Delta}'(x_i)\}_{i=1}^{p_{max}}$. Note that we use interchangably $x_i$ and $x(t)$ as discrete representation of continuous state vector for DNN training. The main purpose of DNN in the adaptive network is to extract relevant features of the system uncertainties, which otherwise is very tedious to obtain without the limits on the domain of operation.

We also demonstrate empirically, that the DNN features trained over past i.i.d representative data retains the memory of the past instances and can be used as the frozen feed-forward network over similar reference tracking tasks without loss of the guaranteed tracking performance.

\subsection{Method for Recording Data using MRGeN for DNN Training}
\label{subsection-buffer}
In statistical inference, implicitly or explicitly one always assume that the training set $Z^M = \{x_i, y_i\}_{i=1}^M$ is composed on M-input-target tuples that are independently drawn from buffer $\mathcal{B}$ over same joint distribution $P(x,y)$. The i.i.d assumption on the data is required for robustness, consistency of the network training and for bounds on the generalization error \cite{xu2012robustness, vandegeer2009}. In classical generalization proofs one such condition is that $\frac{1}{p_{max}}\mathbb{X}^T\mathbb{X} \to \gamma$ as ${p_{max}} \to \infty$, where $\mathbb{X}$ denotes the design matrix with rows $\Phi_i^T$. The i.i.d assumption implies the above condition is fulfilled and hence is sufficient but not necessary condition for consistency and error bound for generative modeling. 

The key capability brought about by DMRAC is a relevant feature extraction from the data. Feature extraction in DNN is achieved by using recorded data concurrently with current data. The recorded data include the state $x_i$, feature vector $\Phi(x_i)$ and associated D-MRGeN estimate of the uncertainty ${\Delta}'(x_i)$. For a given $\zeta_{tol} \in \mathbb{R}_+$ a simple way to select the instantaneous data point $\{x_i, \Delta'(x_i)\}$ for recording is to required to satisfy following condition
\begin{equation}
    \gamma_i = \frac{\|\Phi(x_i) - \Phi_p\|^2}{\|\Phi(x_i)\|} \geq \zeta_{tol}
    \label{eq:kernel_test}
\end{equation}
Where the index $p$ is over the data points in buffer $\mathcal{B}$. The above method ascertains only those data points are selected for recording that are sufficiently different from all other previously recorded data points in the buffer. Since the buffer $\mathcal{B}$ is of finite dimension, the data is stored in a cyclic manner. As the number of data points reaches the buffer budget, a new data is added only upon one existing data point is removed such that the singular value of the buffer is maximized. The singular value maximization approach for the training data buffer update is provided in \cite{5991481}.
% Further, $Z^M$ are drawn uniformly over buffer $\mathcal{B}$ for DNN training as detailed in the previous subsection.
\begin{algorithm}[h!]
    \caption{D-MRAC Controller Training}
    \label{alg:DMRAC}
    \begin{algorithmic}[1]
        \STATE {\bfseries Input:} $\Gamma, \eta, \zeta_{tol}, p_{max}$
        \WHILE{New measurements are available}
        \STATE Update the D-MRGeN weights $W$ using Eq:\eqref{eq:18}
        \STATE Compute $y_{\tau+1} = \hat{W}^T\Phi(x_{\tau+1})$
        \STATE Given $x_{\tau+1}$ compute $\gamma_{\tau+1}$ by Eq-\eqref{eq:kernel_test}.
                \IF{$\gamma_{\tau+1} \geqslant \zeta_{tol}$}
                \STATE Update $\mathcal{B}:\boldsymbol{Z}(:) = \{x_{\tau+1}, y_{\tau+1}\}$ and $\mathbb{X}: \Phi(x_{\tau+1})$
                	\IF{$|\mathcal{B}| > p_{max}$}
                    \STATE Delete element in $\mathcal{B}$ by SVD maximization \cite{5991481}
                    \ENDIF
        	    \ENDIF
    	    \IF{$|\mathcal{B}| \geq M$}
    	    \STATE Sample a mini-batch of data $\boldsymbol{Z}^M \subset \mathcal{B}$
    	    \STATE Train the DNN network over mini-batch data using Eq-\eqref{SGD2}
    	    \STATE Update the feature vector $\Phi$ for D-MRGeN network
    	    \ENDIF
        \ENDWHILE
        \end{algorithmic}
\end{algorithm}

\section{Sample Complexity and Stability Analysis for DMRAC}
In this section, we present the sample complexity results, generalization error bounds and stability guarantee proof for DMRAC. We show that DMRAC controller is characterized by the memory of the features learned over previously observed training data. We further demonstrate in simulation that when a trained DMRAC is used as a feed-forward network with frozen weights, can still produce bounded tracking performance on reference tracking tasks that are related but reasonably different from those seen during network training.  We ascribe this property of DMRAC to the very low generalization error bounds of the DNN. We will prove this property in two steps. Firstly we will prove the bound on the generalization error of DNN using Lyapunov theory such that we achieve an asymptotic convergence in tracking error. Further, we will show information theoretically the lower bound on the number of independent samples we need to train through before we can claim the DNN generalization error is well below a determined lower level given by Lyapunov analysis.
% In this section, we introduce the sample complexity, bounds on the generalization error and stability theory for DMRAC. We show in this paper the DMRAC controller is characterized by the memory over previously observed data and demonstrate in simulation that when DMRAC is used as a feed-forward network with frozen weight (no learning) can still produce bounded tracking performance results on reference tracking tasks that are reasonably different from those seen during network training. We ascribe this property of DMRAC to the very low generalization error bounds of the DNN. We will prove this property in two steps. Firstly we will prove the bound on the generalization error of DNN using Lyapunov theory such that we achieve an asymptotic convergence in tracking error. Further, we will show information theoretically the lower bound on the number of independent datasets we need to train through before we can claim the DNN generalization error is well below a determined lower level given by Lyapunov analysis.
\subsection{Stability Analysis}
The generalization error of a machine learning model is defined as the difference between the empirical loss of the training set and the expected loss of test set \cite{2018arXiv180801174J}. This measure represents the ability of the trained model to generalize well from the learning data to new unseen data, thereby being able to extrapolate from training data to new test data. Hence generalization error can be defined as
\begin{equation}
    \hat{\Delta}(x) - f_{\boldsymbol{\theta}}(x) \leqslant \epsilon
\end{equation}
Using the DMRAC (as frozen network) controller in \eqref{eq:total_Controller} and using systems \eqref{eq:0} we can write the system dynamics as
\begin{equation}
    \dot x(t) = Ax(t) + B(-Kx(t) + K_rr(t) -f_{\boldsymbol{\theta}}(x(t)) + \Delta(x))
\end{equation}
We can simplify the above equation as
\begin{equation}
    \dot x(t) = A_{rm}x(t) + B_{rm}r(t)+B(\Delta(x)-f_{\boldsymbol{\theta}}(x(t)))
\end{equation}
Adding and subtracting the term ${\Delta}'(x)$ in above expression and using the training and generalization error definitions we can write,
\begin{eqnarray}
\dot x(t) &=& A_{rm}x(t) + B_{rm}r(t)\\
&&+B(\Delta(x)-{\Delta}'(x(t))+{\Delta}'(x(t))-f_{\boldsymbol{\theta}}(x(t))) \nonumber
\end{eqnarray}
The term $\left(\Delta(x)-{\Delta}'(x(t))\right)$ is the D-MRGeN training error and $\left({\Delta}'(x(t))-f_{\boldsymbol{\theta}}(x(t))\right)$ is the generalization error of the DMRAC DNN network. For simplicity of analysis we assume the training error is zero, this assumption is not very restrictive since training error can be made arbitrarily small by tuning network architecture and training epochs. The reference tracking error dynamics can be written as,
\begin{equation}
\dot e(t) = A_{rm}e(t)  +  \epsilon
\label{eq:DMRAC_error_dynamics}
\end{equation}
To analyze the asymptotic tracking performance of the error dynamics under DMRAC controller we can define a Lyapunov candidate function as $V(e) = e^TPe$ and its time derivative along the error dynamics \eqref{eq:DMRAC_error_dynamics} can be written as
\begin{equation}
    \dot V(e) = -e^TQe + 2\epsilon Pe
\end{equation}
where $Q$ is solution for the Lyaunov equation $A_{rm}^TP + PA_{rm} = -Q$. To satisfy the condition $\dot V(e) < 0$ we get the following upper bound on generalization error,
\begin{equation}
    \|\epsilon\| < \frac{\lambda_{max}(Q)\|e\|}{\lambda_{min}(P)}
    \label{eq:generalization_bound}
\end{equation}
The idea is, that if the DNN produces a generalization error lower than the specified bound \eqref{eq:generalization_bound}, then we can claim Lyanpunov stability of the system under DMRAC controller.
% \vspace{-2mm}
\subsection{Sample Complexity of DMRAC}
In this section, we will study the sample complexity results from computational theory and show that when applied to a network learning real-valued functions the number of training samples grows at least linearly with the number of tunable parameters to achieve specified generalization error.
\begin{theorem}
Suppose a neural network with arbitrary activation functions and an output that takes values in $[-1,1]$. Let $\mathcal{H}$ be the hypothesis class characterized by N-weights and each weight represented using k-bits. Then any squared error minimization (SEM) algorithm $\mathcal{A}$ over $\mathcal{H}$, to achieve a generalization error \eqref{eq:generalization_bound} admits a sample complexity bounded as follows
\begin{equation}
    m_{\mathcal{A}}(\epsilon, \delta) \leqslant \frac{1}{\epsilon^2} \left(kN\ln2 + \ln\left(\frac{2}{\delta}\right)\right)
\end{equation}
where $N$ is total number of tunable weights in the DNN.
\end{theorem}
\begin{proof}
Let $\mathcal{H}$ be finite hypothesis class of function mapping s.t $\mathcal{H}: \mathcal{X} \to [-1,1] \in \mathbb{R}^m$ and $\mathcal{A}$ is SEM algorithm for $\mathcal{H}$. Then by Hoeffding inequality for any fixed $f_{\boldsymbol{\theta}} \in \mathcal{H}$ the following event holds with a small probability $\delta$
\begin{eqnarray}
&&P^m\{|L(\boldsymbol{Z, \theta}) - \mathbb{E}_P(\ell(\boldsymbol{Z, \theta}))| \geq \epsilon\}\\
&=& P^m\left\{\left|\sum_{i=1}^m \ell(\boldsymbol{Z, \theta}) - m\mathbb{E}_P(\ell(\boldsymbol{Z, \theta}))\right| \geq m\epsilon\right\}\\
&\leq& 2e^{-\epsilon^2m/2}
\end{eqnarray}
Hence\vspace{-5mm}
\begin{eqnarray}
&&P^m\{ \forall f_{\boldsymbol{\theta}} \in \mathcal{H}, | \left|L(\boldsymbol{Z, \theta}) - \mathbb{E}_P(\ell(\boldsymbol{Z, \theta}))\right| \geq \epsilon\} \nonumber \\
&\leq& 2|\mathcal{H}|e^{-\epsilon^2m/2} = \delta
\label{eq:sample_complexity}
\end{eqnarray}
We note that the total number of possible states that is assigned to the weights is $\left(2^k\right)^N$ since there are $2^k$ possibilities for each weights. Therefore $\mathcal{H}$ is finite and $|\mathcal{H}| \leq 2^{kN}$. The result follows immediately from simplifying Eq-\eqref{eq:sample_complexity}.
\end{proof}

\section{Simulations}
\label{results}
In this section, we will evaluate the presented DMRAC adaptive controller using a 6-DOF Quadrotor model for the reference trajectory tracking problem. The quadrotor model is completely described by 12 states, three position, and velocity in the North-East-Down reference frame and three body angles and angular velocities. The full description of the dynamic behavior of a Quadrotor is beyond the scope of this paper, and interested readers can refer to \cite{joshi2017robust} and references therein.

The control law designed treats the moments and forces on the vehicle due to unknown true inertia/mass of the vehicle and moments due to aerodynamic forces of the crosswind, as the unmodeled uncertainty terms and are captured online through DNN adaptive element. The outer-loop control of the quadrotor is achieved through Dynamic Inversion (DI) controller, and we use DMRAC for the inner-loop attitude control. A simple wind model with a boundary layer effect is used to simulate the effect of crosswind on the vehicle. 

A second-order reference model with natural frequency $4 rad/s$ and damping ratio of $0.5$ is used. Further stochasticity is added to the system by adding Gaussian white noise to the states with a variance of $\omega_n = 0.01$. The simulation runs for $150$secs and uses time step of $0.05s$. The maximum number of points ($p_{max}$) to be stored in buffer $\mathcal{B}$ is arbitrarily set to $250$, and SVD maximization algorithm is used to cyclically update $\mathcal{B}$ when the budget is reached, for details refer \cite{5991481}.  

The controller is designed to track a stable reference commands $r(t)$. The goal of the experiment is to evaluate the tracking performance of the proposed DMRAC controller on the system with uncertainties over an unknown domain of operation. The learning rate for D-MRGeN network and DMRAC-DNN networks are chosen to be $\Gamma = 0.5I_{6 \times 6}$ and $\eta = 0.01$. The DNN network is composed of 2 hidden layers with $200,100$ neurons and with tan-sigmoid activations, and output layer with linear activation. We use ``Levenberg-Marquardt backpropagation'' \cite{yu2011levenberg} for updating DNN weights over $100$ epochs. Tolerance threshold for kernel independence test is selected to be $\zeta_{tol} = 0.2$ for updating the buffer $\mathcal{B}$.

Figure-\ref{fig:plot_1} and Fig-\ref{fig:plot_2} show the closed loop system performance in tracking the reference signal for DMRAC controller and learning retention when used as the feed-forward network on a similar trajectory (Circular) with no learning. We demonstrate the proposed DMRAC controller under uncertainty and without domain information is successful in producing desired reference tracking. Since DMRAC, unlike traditional MRAC, uses DNN for uncertainty estimation is hence capable of retaining the past learning and thereby can be used in tasks with similar features without active online adaptation Fig-\ref{fig:plot_2}. Whereas traditional MRAC which is ``pointwise in time" learning algorithm and cannot generalize across tasks. The presented controller achieves tighter tracking with smaller tracking error in both outer and inner loop states as shown in Fig-\ref{fig:plot_2} and Fig-\ref{fig:plot_3} in both with adaptation and as a feed-forward adaptive network without adaptation.  Figure-\ref{fig:plot_4} demonstrate the DNN learning performance vs epochs. The Training, Testing and Validation error over the data buffer for DNN, demonstrate the network performance in learning a model of the system uncertainties and its generalization capabilities over unseen test data. 
\begin{figure*}[tbh]
    \centering
	\begin{subfigure}{0.9\columnwidth}
		%\centering
		\includegraphics[width=1.05\textwidth, height = 0.6\textwidth]{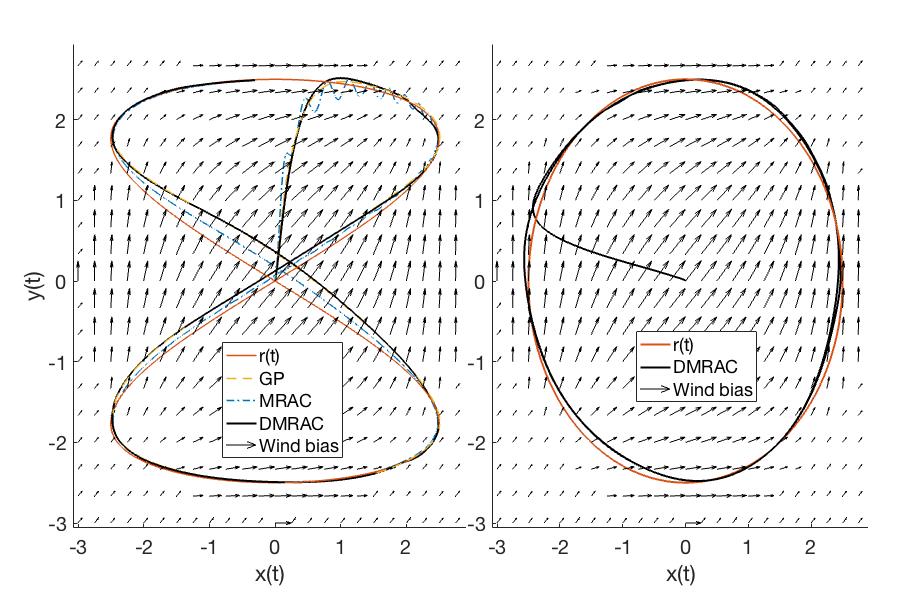}
		\vspace{-0.2in}
		\caption{}
		\label{fig:plot_1}
	\end{subfigure}
	\begin{subfigure}{1.1\columnwidth}
		%\centering
		\includegraphics[width=\textwidth, height = 0.5\textwidth]{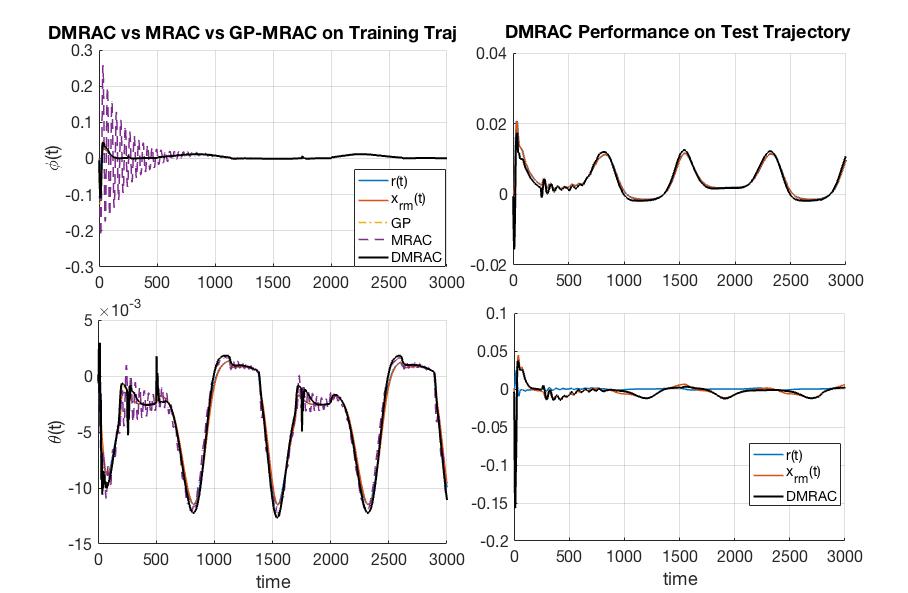}
		\vspace{-0.2in}
		\caption{}
		\label{fig:plot_2}
	\end{subfigure}
	\vspace{-0.12in}
	\caption{DMRAC Controller Evaluation on 6DOF Quadrotor dynamics model (a) DMRAC vs MRAC vs GP-MRAC Controllers on quadrotor trajectory tracking with active learning and DMRAC as frozen feed-forward network (Circular Trajectory) to test network generalization (b) Closed-loop system response in roll rate $\phi(t)$ and Pitch $\theta(t)$}
	\label{fig:grid_world}
	\vspace{-4mm}
\end{figure*}
\begin{figure*}[tbh]
    \centering
	\begin{subfigure}{1.2\columnwidth}
		%\centering
		\includegraphics[width=\textwidth, height = 0.5\textwidth]{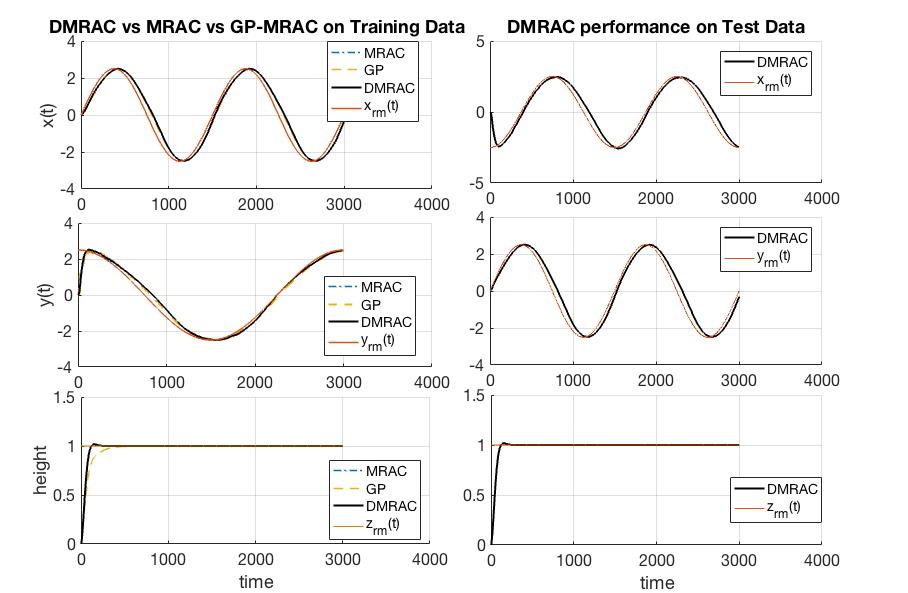}
		\vspace{-0.2in}
		\caption{}
		\label{fig:plot_3}
	\end{subfigure}
	\begin{subfigure}{0.80\columnwidth}
		%\centering
		\includegraphics[width=1.1\textwidth]{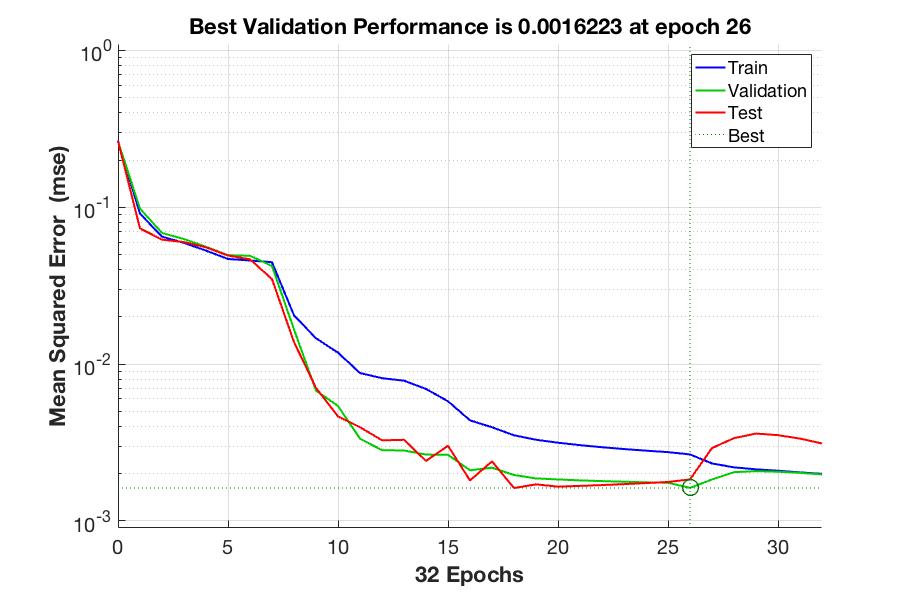}
		\vspace{-0.2in}
		\caption{}
		\label{fig:plot_4}
	\end{subfigure}
	\vspace{-0.12in}
	\caption{(a) Position Tracking performance of DMRAC vs MRAC vs GP-MRAC controller with active learning and Learning retention test over Circular Trajectory for DMRAC (b) DNN Training, Test and Validation performance.}
	\label{fig:grid_world}
	\vspace{-5mm}
\end{figure*}
% \begin{figure}
% \centering
% \includegraphics[width=0.7\linewidth]{Fig/Position_Plot}
% \caption{Closed-loop system response $\theta(t)$ of wing rock aircraft dynamics model
% with the proposed adaptive controller GP-MRGeN and MRAC
% 	adaptive controller.}
% \label{fig:plot_1}
% \end{figure}
% \begin{figure}
% \centering
% \includegraphics[width=0.7\linewidth]{Fig/Velocity_Plot}
% \caption{Closed-loop system response $p(t)$ of wing rock aircraft dynamics model
% with the proposed adaptive controller GP-MRGeN and MRAC
% 	adaptive controller.}
% \label{fig:plot_2}
% \end{figure}
% \begin{figure}
% 	\centering
% 	\includegraphics[width=0.7\linewidth]{Fig/State_Error}
% 	\caption{Tracking Error for the proposed adaptive controller GP-MRGeN and MRAC
% 	adaptive controller (a) Roll angle Error (b) Roll Rate Error}
% 	\label{fig:plot_4}
% \end{figure}
% \begin{figure}
% 	\centering
% 	\includegraphics[width=0.7\linewidth]{Fig/Control_Plot}
% 	\caption{Total Control input u(t) for the desired reference model tracking for GP-MRGeN and MRAC Controllers}
% 	\label{fig:plot_5}
% \end{figure}
% \begin{figure}
% 	\centering
% 	\includegraphics[width=0.7\linewidth]{Fig/Disturbance_plot}
% 	\caption{Uncertainty approximation for the proposed adaptive controller GP-MRGeN and MRAC adaptive controller. The blobs indicate the time step at which generative model is updated and queried for GP training.}
% 	\label{fig:plot_3}
% \end{figure} 
% \begin{figure}
% 	\centering
% 	\includegraphics[width=0.9\linewidth]{Fig/Phase_plot}
% 	\caption{RBF Network Weights for uncertainty approximation}
% 	\label{fig:plot_3}
% \end{figure}

\section{Conclusion}
\label{conclusions}
In this paper, we presented a DMRAC adaptive controller using model reference generative network architecture to address the issue of feature design in unstructured uncertainty. The proposed controller uses DNN to model significant uncertainties without knowledge of the system's domain of operation. We provide theoretical proofs of the controller generalizing capability over unseen data points and boundedness properties of the tracking error. Numerical simulations with 6-DOF quadrotor model demonstrate the controller performance, in achieving reference model tracking in the presence of significant matched uncertainties and also learning retention when used as a feed-forward adaptive network on similar but unseen new tasks. Thereby we claim DMRAC is a highly powerful architecture for high-performance control of nonlinear systems with robustness and long-term learning properties.
\vspace{-1mm}

\bibliographystyle{unsrt}
\bibliography{TAC_Note_Unmatched_Uncertainty_references}

\end{document}